\documentclass[onecolumn, draftcls]{IEEEtran}
\usepackage{amsmath,amsfonts}
\usepackage{amsthm}
\usepackage{algorithmic}
\usepackage{algorithm}
\usepackage{graphicx}
\usepackage{float}
\usepackage{cite}
\usepackage{csquotes}
\usepackage{enumerate}

\newcommand{\identityf}[1]{\mathbf 1_{\{#1\}}}

\newtheorem{theorem}{\textbf{Theorem}}
\newtheorem{lemma}{\textbf{Lemma}}
\begin{document}

\title{Heterogeneous Multi-Player Multi-Armed Bandits Robust To Adversarial Attacks}

\author{%
  \IEEEauthorblockN{Akshayaa Magesh,
                    Venugopal V.~Veeravalli}
  \IEEEauthorblockA{%
                   ECE \& CSL, The Grainger College of Engineering, University of Illinois at Urbana-Champaign, Champaign, IL, USA \\
                    \{amagesh2, vvv\}@illinois.edu}
}

% \markboth{Journal of \LaTeX\ Class Files,~Vol.~1, No.~2, December~2023}%

% \IEEEpubid{0000--0000~\copyright~2023 IEEE}

\maketitle

\begin{abstract}
We consider a multi-player multi-armed bandit setting in the presence of adversaries that attempt to negatively affect the rewards received by the players in the system. The reward distributions for any given arm are heterogeneous across the players. In the event of a collision (more than one player choosing the same arm), all the colliding users receive zero rewards. The adversaries use collisions to affect the rewards received by the players, i.e., if an adversary attacks an arm, any player choosing that arm will receive zero reward. At any time step, the adversaries may attack more than one arm. It is assumed that the players in the system do not deviate from 
   a pre-determined policy used by all the players, and that the probability that none of the arms face adversarial attacks is strictly positive at every time step. In order to combat the adversarial attacks, the players are allowed to communicate using a single bit for $O(\log T)$ time units, where $T$ is the time horizon, and each player can only observe their own actions and rewards at all time steps. We propose a {policy that is used by all the players, which} achieves near order optimal regret of order $O(\log^{1+\delta}T + W)$, where $W$ is total number of time units for which there was an adversarial attack on at least one arm.
\end{abstract}

\section{Introduction}
{T}{he} multi-armed bandit (MAB) is a well-studied model for sequential decision-making problems.  The work in \cite{lairob} formalized the stochastic MAB {problem}, and provided a lower bound on the average regret of order $\Omega(\log T)$ for time horizon $T$. Recently, there has been a growing interest in the study of \emph{multi-player} MAB settings, where there are $K$ agents simultaneously pulling the arms at each time step. The average system regret is defined with respect to the optimal assignment of arms that maximizes the sum of the expected rewards of all players.  
The event of multiple players pulling the same arm simultaneously is commonly referred to as a \emph{collision}, and {generally} leads to the players receiving reduced or zero rewards. Thus, while designing policies for the multi-player MAB setting, in addition to balancing the exploration-exploitation trade-off, it is important to control the number of collisions that occur.

Multi-player bandit models can be broadly classified into centralized and decentralized {models}. In the centralized setting, there exists a central controller that can coordinate the actions of all the players, and the multi-player problem can be reduced to a single agent problem. The decentralized system, where there is no central controller, can further be categorized based on the interaction between the agents. Some works assume that sensing occurs at the level of every individual player, i.e., the players have information on whether an arm is occupied before pulling it \cite{dileep}. There are works that assume that players can observe only their own actions and rewards, and can have constrained communication at certain time units \cite{evirgen, avnercomm}. A fully decentralized scenario, where players cannot communicate with each other in any manner, is studied \cite{kaufmann, asilomar, got, bistritz2020game, magesh2021decentralized}. 
Multi-player stochastic MABs in the decentralized setting are particularly relevant to dynamic spectrum access systems 
 \cite{biglieri2013principles}.
%{\color{red} VV: what we need here is a reference to a tutorial article or book on cognitive radio.}
%\cite{dileep, mega, mc, anandkumar, kaufmann, hanawal, perchet, besson, asilomar, got, vvv, evirgen, avnercomm}. 
In these systems, the finite number of channels representing different frequency bands are treated as arms, the users in the network are treated as players, and the data rates received from the channels can be interpreted as rewards. 

In most of the prior work in the decentralized setting, the assumption is made that the reward distribution for any arm is the same (homogeneous) for all players. This setting was first considered in \cite{mega}, where an algorithm that combines a probabilistic $\epsilon$-greedy algorithm with a collision avoiding mechanism inspired by the ALOHA protocol is proposed. The decentralized setting with homogeneous reward distributions across the players was also studied in \cite{mc, anandkumar, perchet,besson}. In cognitive radio and uncoordinated dynamic spectrum access networks, the users are usually not colocated physically, and therefore the reward distributions for a given arm may be heterogeneous across users. There have been a few works that study such a heterogeneous setting. In \cite{dileep, hanawal}, the heterogeneous setting is studied under the assumption that the players are capable of sensing. A fully decentralized heterogeneous setting is studied in \cite{kaufmann, asilomar, got, bistritz2020game, magesh2021decentralized}, where players can observe only their own actions and rewards. In \cite{kaufmann, asilomar}, it is assumed that in the event of a collision, the colliding players get zero rewards, and the idea of forced collisions  is used to enable the players to communicate with one another and settle on the optimal assignment of arms. 
%While the work in \cite{kaufmann} provides guarantees of regret of order $O(\log T)$ only in the event of a unique optimal matching, \cite{asilomar} provides guarantees of logarithmic regret even in the case of multiple optimal matchings. 
The works in \cite{bistritz2020game} and \cite{magesh2021decentralized} take a game-theoretic approach adapted from \cite{marden}, and provide near order-optimal regret guarantees in the zero rewards on collisions and non-zero rewards on collisions, respectively.

The above discussed works assume that all the players in the system follow  
{a pre-determined protocol}, and that there are no adversarial attacks against the players. However, for instance, in the scenario of uncoordinated spectrum access, there might be adversarial agents disrupting the rewards received by the players. Under such attacks, it is possible for players following the proposed protocols from the above works to incur linear regret \cite{mahesh2022multi}. Thus, it is of interest to develop algorithms for the multi-player MAB setting {that are} robust to adversarial attacks. The work in \cite{vial2021robust} considers the robust multi-player setting, where each player has access to a separate instance of a bandit, and all the bandit instances share the same reward distributions. The work in \cite{mahesh2022multi} studies a multi-player setting with a shared bandit instance among all players under the homogeneous reward distribution setting. The work in \cite{boursier2020selfish}  considers a setting where there are selfish players who aim to maximize their own rewards. 

\textbf{Our Contributions: } We consider a setting where at all time steps, each player can observe only their own actions and rewards. In order to combat adversarial attacks, we employ the use of one-bit communication rounds for $O(\log T)$ time units. The reward distributions of the arms across the players are heterogeneous. We consider the setting of zero rewards on collisions. We consider a strong adversarial model, where the adversaries use collisions to affect the rewards received by the players, i.e., if an adversary attacks a given arm, any player choosing that arm will receive zero reward. Thus, a player cannot differentiate between the adversarial attack and a collision from another player in the system. At any time {step}, the adversaries may attack more than one arm. The probability that there are no adversarial attacks on a given arm is strictly positive at every time step. Under this setting, we propose an algorithm that achieves regret that is nearly order-optimal with respect to the time horizon and additive in the number of time units for which adversarial attacks occur on any arm.

\section{System Model}\label{sec:system_model}

Consider a system with $K$ agents. The action space of each agent is the set of $M$ arms, i.e., $\mathcal{A}_k = [M] = \{1, \ldots, M\}$. Let the time horizon be denoted by $T$ and the action taken at time $t$ by agent $k$ be denoted by $a_{t,k}$. The action profile $\mathbf{a}_t$ is defined as $\mathbf{a}_t = [a_{t,1},..., a_{t,K}]$.  
We assume the reward distribution of each arm to have support $[0,1]$. 
Let $n(a_{t,k})$ denote the number of players playing arm $a_{t,k}$ (including player $k$). In the absence of any adversarial attacks, the reward received by player $k$ playing arm $m$, which is played by a total of $n(m)$ players (including player $k$) is denoted by $r_k(m,n(m))$. The reward is drawn from a distribution with mean 
\begin{equation}
 \tilde{\mu}_{k}(m,n(m)) = E[{\tilde{r}_k(m,n(m))}],   
\end{equation}
where the expectation is over the stochasticity of the rewards for given $k$, $m$ and $n(m)$.
Note that
\[
\tilde{\mu}_{k}(m,n(m)) = 0~~\text{for all~} n(m) \geq 2,
\]
i.e., colliding players receive zero rewards. 

The adversarial attacks at time $t$ are described by a binary valued vector $\mathbf{w_t} = [w_{t,1}, \ldots, w_{t,M}] \in \{0,1\}^M$, where $w_{t,m} = 1$  denotes that there is an adversary attacking arm $m$ {at time $t$}, and $w_{t,m} = 0$ denotes the absence of an attack on arm $m$ {at time $t$}. We assume that at each time $t$, the probability that there is no adversary attacking arm $m$ is bounded away from zero, i.e., $P\{w_{t,m} = 0\} > 0$, {for all $m \in [M]$}. We denote the reward received by player $k$ playing arm $m$ under an adversarial attack $\mathbf{w}$ by $r_k(m,n(m), \mathbf{w})$. If $w_{t,m} = 1$, then any player choosing arm $m$ receives zero reward. Thus, an adversarial attack on arm $m$ mimics a collision, and any player playing arm $m$ cannot differentiate between the adversarial attack and a collision from another player in the system.  If $w_{t,m} = 0$, players playing arm $m$ receive the same rewards they would receive in the absence of an adversary. The reward is drawn from a distribution with mean
\begin{equation}
    \mu_k(m,n(m), \mathbf{w}) = 
    \begin{cases}
    \mu_k(m,n(m)) \;\text{ if } w_m = 0, \\
    0   \qquad \qquad  \;\;\;\;  \text{ if } w_m = 1.
    \end{cases}
\end{equation}

The action space $\mathcal{A}$ of the players is the product space of the individual action spaces, \textit{i.e.}, $\mathcal{A} = \Pi_{k=1}^K \mathcal{A}_k$. Let $\mathbf{a^*} \in \mathcal{A} $ be such that
\begin{equation}\label{eq:optimal_action_profile}
    \mathbf{a^*} \in \mathop{\arg\max}_{\mathbf{a} \in \mathcal{A}} \sum_{k=1}^K \mu_{k}(a_k, n(a_k) , \mathbf{w^0} ),
\end{equation} 
where $\mathbf{w^0} = [0, \ldots, 0]$ (i.e., no adversarial attacks). Note that if the number of players $K$ is greater than the number of arms $M$, then some players will receive zero rewards under the optimal action profile. {Therefore}, we assume $K$ to be less than $M$, {which means}
\begin{equation}
    \mathbf{a^*} \in \mathop{\arg\max}_{\mathbf{a} \in \mathcal{A}} \sum_{k=1}^K \mu_{k}(a_k, 1),
\end{equation} 

In this work, we restrict our attention to the case where there is a unique\footnote{The assumption of a unique optimal matching has been made and justified in previous works, see, e.g., \cite{got, bistritz2020game}. This assumption is needed to establish the convergence of the proposed algorithm to the optimal action profile.} {\em optimal} matching $\mathbf{a^*}$. Let $J_1 = \sum_{k=1}^K \mu_{k}(a^*_k,n(a^*_k), \mathbf{w^0})$ be the system reward for the optimal matching, and $J_2$ the system reward for the second optimal matching. Define  $\Delta = \frac{J_1 - J_2}{2 M}$.
{We assume that the players do \emph{not} know $\Delta$.} 

The expected regret $R(T)$ during a time horizon $T$ is: 
\begin{equation}
  T \sum_{k=1}^K \mu_{k}(a^*_k, n(a^*_k), \mathbf{w^0}) - E[{\sum_{t=1}^T \sum_{k=1}^K \mu_{k}(a_{t,k},n(a_{t,k}), \mathbf{w_t}) }]  
\end{equation}
where the expectation is over the actions of the players and adversarial attacks. 

In order for the players to get estimates of the mean rewards of the arms, we assume that the players have unique IDs at the beginning of the algorithm, as in previous related works \cite{boutilier1996planning, magesh2021decentralized}.
%\vvv{[VVV: Give citations to two previous works, one being our IT paper and a second one from another group]}.

In order to combat the adversarial attacks under a heterogeneous reward distribution setting, we employ one-bit communication rounds between the players at certain time {steps} which are unaffected by the adversaries. In a time horizon of $T$, we make use of $O(\log T)$ such one bit communication rounds. Apart from these time {steps}, it is assumed that players can only observe only their own actions and rewards. 

\section{Algorithm}

\begin{algorithm}[t]
   \caption{Policy for each player $k$}
   \label{alg:main}
\begin{algorithmic}
   \STATE {\bfseries Initialization:} Set $\hat{\mu}_{k}(m,1) = 0 $ for all $k \in [K]$ and $m \in [M]$. Let $L_T$ be the last epoch with time horizon $T$. Parameters $\delta>0$ and $\epsilon \in (0,1)$ are provided as inputs.
   \FOR{epoch $\ell = 1$ {\bfseries to} $L_T$}
   \STATE \textbf{Exploration phase:} Pull arms in order of unique IDs until $T_0 \ell^\delta$ reward observations without adversarial attacks
   \STATE \textbf{Matching phase:} Run Algorithm \ref{alg:matching} with input $\ell$ for $\tau_\ell = {c_2 \ell^{ \delta}}$ time units.  Count the number of `plays' where each action $m \in [M]$ was played that resulted in player $j$ being content: $$W^\ell(k,m) = \sum_{t = 1}^{\tau_\ell} \identityf{(a_{t,k} = m, S_{t,k} = C)}$$
   \STATE \textbf{Exploitation phase:} For $c_3 2^\ell$ time units, play the action played most frequently from epochs $\lceil \frac{\ell}{2}\rceil$ to $\ell$ that resulted in player $k$ being content: 
   $$a_j = \mathop{\arg\max}_m \sum_{i = \lceil \frac{\ell}{2}\rceil}^\ell W^i(k,m)$$
   \ENDFOR
\end{algorithmic}
\end{algorithm}

\begin{algorithm} 
   \caption{Matching phase algorithm}
   \label{alg:matching}
\begin{algorithmic}
   \STATE {\bfseries Initialization:} Let $\kappa > M$ and $\beta < \kappa$. Denote by $\hat{Z}_{t,k}$ the state of player $k$ at time $t$, and set  $\hat{Z}_{1,k} = [\Bar{a}_{1,k}, \Bar{u}_{1,k}, S_{1,k}],$ where $\Bar{a}_{1,k}\mathop{\sim}\limits^{\text{unif}}[M]$, $\Bar{u}_{1,k} = 0$ and $S_{1,k} = D$. Set $\tau_\ell = {c_2 \ell^\delta}$.  Input parameter $\epsilon \in (0,1)$. 
   \FOR{time $t = 1$ {\bfseries to} $\tau_\ell$}
   \STATE \textbullet~\textbf{Action dynamics:}
   \STATE If $S_{t,j} = C$, set action $a_{t,k}$ as:
   \[
   a_{t,k} = \begin{cases} \Bar{a}_{t,k} &\text{with prob} \;\; 1 - \epsilon^\kappa \\ a \in [M]\setminus{\Bar{a}_{t,k}}  &\text{with uniform prob} \;\; \frac{\epsilon^\kappa}{M-1}. \end{cases}
   \]
   \STATE If $S_{t,k} = D$, action $a_{t,k}$ is chosen uniformly from $[M]$.
   \STATE \textbullet~\textbf{Get utility:} Upon choosing action $a_{t,k}$, receive reward $r_k(a_{t,k},n(a_{t,k}), \mathbf{w_t})$. If $r_k(a_{t,k},n(a_{t,k}), \mathbf{w_t}) = 0$, the utility of the player ${u}_{t,k} = 0$. Else the utility ${u}_{t,k}$ is: 
   \[
   {u}_{t,k} = \hat{\mu}_{k}(a_{t,k},1).
   \]
   \STATE \textbullet~\textbf{State dynamics:} 
   \STATE If $S_{t,k} = C$ and $a_{t,k} = \Bar{a}_{t,k}$, set: 
   \begin{equation}\label{eq:content_update}
       {Z}_{t+1,k} = {Z}_{t,k}
   \end{equation}
   \STATE If $S_{t,k} = C$ and $a_{t,k} \neq \Bar{a}_{t,k}$, or $S_{t,k} = D$, set:
   \begin{equation}\label{state}
       {Z}_{t+1,k} = \begin{cases} [a_{t,k}, {u}_{t,k}, C] &\text{with prob} \;\; \epsilon^{1 - {u}_{t,k}} \\ [a_{t,k}, {u}_{t,k}, D] &\text{with prob} \;\; 1 - \epsilon^{1 - {u}_{t,k}} \end{cases} 
   \end{equation}
  \STATE \textbullet~\textbf{Synchronize moods:}
  \STATE Get $S_{t,k}$ from all $K$ players 
  \STATE If $S_{t+1,k} = C$ and $\prod\limits_{k=1}^K \identityf{S_{t,k} = C} = 1$, retain mood as $S_{t+1,k} = C$
  \STATE If $S_{t+1,k} = D$, retain mood as $S_{t+1,k} = D$
  \STATE If $S_{t+1,k} = C$ and $\prod\limits_{k=1}^K \identityf{S_{t,k} = C} = 0$, update mood as
  \begin{equation}\label{eq:update_mood}
      {S}_{t+1,k} = \begin{cases} C &\text{with prob} \;\; \epsilon^{\beta} \\ D &\text{with prob} \;\; 1 - \epsilon^\beta \end{cases}
  \end{equation}
   \ENDFOR
\end{algorithmic}
\end{algorithm}

The proposed policy for each player $k$ in the multi-player multi-armed bandit setting with heterogeneous reward distributions and adversarial attacks in presented in Algorithm \ref{alg:main}. The algorithm proceeds in epochs since we do not assume the knowledge of the time horizon $T$. The parameters $\delta$, $\epsilon$ and $\beta$ are inputs to the algorithm and further details are provided in the following sections. Let $L_T$ denote the number of epochs in time horizon $T$. Each epoch $\ell$ has three phases: Exploration, Matching and Exploitation. 

The exploration phase is for each player to obtain estimates of the mean rewards $\hat{\mu}_k(m,1, \mathbf{w^0})$ of each arm $m \in [M]$. 
%Note that $\hat{\mu}_k(m,n, \mathbf{w^0}) = 0$ for $n \geq 2$ for all arms. 
As stated in the previous section, we assume that the players have been assigned unique IDs from $1$ to $K$. In this phase, the players pull the arms in the order of their unique IDs until they have $T_0 \ell^\delta$ reward observations without adversarial attacks, and then use one-bit communication rounds to synchronize moving to explore the next arm. The algorithm is detailed in the Appendix. Note that at every time instant $t$ in this phase, player $k$ is able to identify if there is an adversarial attack on their chosen arm $a_{t,k}$, since the protocol is set up such that there are no collisions among the players. Thus, when a player receives zero reward in the exploration phase, it signifies an adversarial attack on that arm, and they discard that reward observation. The exploration phase in each epoch lasts for $T_\mathrm{exp} = MT_0 \ell^\delta$ time units in the absence of any adversarial attacks. Otherwise, the exploration phase in each epoch lasts for $T_\mathrm{exp} \leq MT_0 \ell^\delta + W^\ell_\mathrm{exp}$ time units, where $W^\ell_\mathrm{exp}$ is the number of time units in the exploration phase in epoch $\ell$ during which there was an adversarial attack on at least one arm. During epoch $\ell$, if the estimated mean rewards obtained at the end of the exploration phase  do not deviate from the mean rewards by more than $\Delta$, the exploration phase of this epoch is \emph{successful}. Since we do not assume that the players have knowledge of $\Delta$, having increasing lengths of the exploration phases by setting $\delta > 0$ ensures that this phase is successful with high probability eventually (for large enough $\ell$). A lower bound on the probability of this event is provided in the Appendix, which will be used in the regret analysis.  

Once the players have estimates of the mean rewards of the arms, they need to converge to the action profile that maximizes the system reward $\mathbf{a^*}$ as defined in \eqref{eq:optimal_action_profile}. This is done in the matching phase of each epoch. Note that though $\mathbf{a^*}$ is defined with respect to the mean rewards in the absence of any adversarial attacks, the players will face adversarial attacks during the matching phase. For an action profile $\mathbf{a}$ and adversarial attack $\mathbf{w}$, the utility received by each player $k$ is defined as:
\begin{equation}
    u_k(\mathbf{a}, \mathbf{w}) = 
    \begin{cases}
        \hat{\mu}_k(a_k,1) &\text{if} \;\; n(a_k) = 1 \text{    and } w_{a_k} = 0 \\
        0 &\text{if} \;\; n(a_k) \geq 2 \text{ 
   or  } \;\; w_{a_k} = 1.
    \end{cases}
\end{equation}
The utility received by player $k$ depends on the actions of all the players and the adversarial attack on arm $a_k$. The matching phase is inspired from the work by \cite{marden}, in which a strategic form game is studied, where players are aware of only their own payoffs (utilities). A decentralized strategy is presented by \cite{marden} that leads to an efficient configuration of the players' actions. The works by \cite{got} and \cite{magesh2021decentralized} apply the strategy proposed by \cite{marden} to setting where there are no adversarial attacks. However, the algorithms proposed by both \cite{got} and \cite{magesh2021decentralized} would incur linear regret in the presence of adversarial attacks for $O(\log T)$ time units. This is because the adversary can attack any subset of arms during the matching phases in \cite{got} and \cite{magesh2021decentralized} for $O(\log T)$ time units which would cause the optimal action profile estimated during the matching phase to be wrong, and regret is incurred during the exploitation phase leading to overall linear regret. In our proposed algorithm (Algorithm \ref{alg:matching}), we adapt the payoff based scheme from \cite{marden} to account for the adversarial attacks, and introduce one bit communication rounds to combat the attacks.  In order to provide near order-optimal regret guarantees, we need to prove that our algorithm leads to an action profile that maximizes the sum of utilities (under $\mathbf{w^0}$) of the players even in the presence of adversarial attacks. This is analyzed in Section \ref{sec:matching_phase}. This phase proceeds for $\tau_\ell = c_2  \ell ^{ \delta}$ time units in epoch $\ell$, where $c_2$ is a constant.  As in the exploration phase, we need increasing lengths of matching phases ($\delta > 0$) to guarantee that the players identify the optimal action profile with high probability.

The action profile identified at the end of the matching phase is played in the exploitation phase for $c_3 2^\ell$ time units, where $c_3$ is a constant. As $\ell$ increases, the players get better estimates of the mean rewards and the probability of identifying the optimal action profile increases. Therefore, the length of the exploitation phase is set to be exponential in $\ell$. 
The constants $c_2$ and  $c_3$ are are chosen to be of the order of $T_e$, the time taken by the exploration phase.

\section{Main Result}\label{sec:main_result}

In this section, we provide the regret guarantees for the proposed algorithm in Algorithm \ref{alg:main}. The proof sketch for the theorem is provided below, and the full proof is provided in the Appendix. 

\begin{theorem} \label{thm:main_thm_mabs}
Given the system model specified in Section \ref{sec:system_model}, the expected regret of the proposed Algorithm \ref{alg:main} for a time-horizon $T$ and some $0< \delta <1 $ is $R(T) = O(\log^{1+\delta} T + W)$, where $W$ is total number of time units for which there was an adversarial attack on at least one arm.
\end{theorem}

\textbf{Proof Sketch}  Let $L_T$ be the last epoch within a  time-horizon of  $T$. The regret incurred during the $L_T$ epochs can be analyzed as the sum of the regret incurred during the three phases of the algorithm. The exploitation phase in epoch $\ell$ of the algorithm lasts for $c_3 2^{\ell}$ time units. Thus, it is easy to see that $L_T < \log{T}$. Let $R_1$, $R_2$ and $R_3$ denote the regret incurred  over $L_T$ epochs, during the exploration phase, the matching phase, and the exploitation phase, respectively. 

Since the exploration phase in each epoch $\ell$ proceeds for at most $MT_0 \ell^\delta + W^\ell_\mathrm{exp}$ time units, $R_1 \leq K M T_0 \log^{1 + \delta} T + K W^\ell_\mathrm{exp}$. Similarly, since the matching phase runs for $\tau_\ell = c_2  \ell^{\delta}$ time units in epoch $\ell$, $R_2 \leq K c_2  \log^{1+\delta}{T}.$ In the exploitation phase, regret is incurred in the following three events:
    \begin{enumerate}
        \item Let $E^\ell$ denote the event that there exists some player $k \in [K]$, arm $m \in [M]$, such that there exists some epoch $i$, with $\lceil \frac{\ell}{2} \rceil \leq i \leq \ell$, such that the estimate of the mean reward $\hat{\mu}_k(m,1)$ obtained after the exploration phase of epoch $i$ satisfies $|\hat{\mu}_k(m,1) - \mu_{k}(m,1)| \geq \Delta$.  
        %Let the probability of this event be $P(E^\ell)$.
        \item Let $F^\ell$ denote the event that given that all the players have $|\mu_{k}(m,1) - \hat{\mu}_{k}(m,1)| \leq \Delta$ for all $m \in [M]$ and all $n \in [N]$ for all epochs $\lceil \frac{\ell}{2} \rceil $ to $ \ell$, the action profile chosen in the matching phase of epoch $\ell$ is not optimal. 
        %Let the probability of this event be $P(F^\ell)$.
        \item There is an adversarial attack on any arm.
    \end{enumerate}

    We can show that for all epochs $\ell \geq \ell_0$ (for a finite $\ell_0$ defined in the Appendix):
 \begin{align}
     P(E^\ell) \leq 2KMe^{-\ell}, 
     P(F^\ell) \leq  e^{-\ell},
 \end{align}
 where the bound on the probability of event $F^\ell$ is provided in Lemma \ref{lem:matching_phase_success_prob}. This bounds the regret in the exploitation phases as $R_3 \leq 2Kc_32^{\ell_0} + \frac{2Kc_3 ( 2KM+ 1)}{e-2} + KW^\ell_\mathrm{exploit}$,  where $KW^\ell_\mathrm{exploit}$ is the number of time units in the exploitation phase of epoch $\ell$ where there was an adversarial attack on at least one arm. Thus 
\begin{equation}
    \begin{split}
    R(T) &= R_1 + R_2 +R_3 \\
    &\leq K M  T_0 \log^{1 + \delta} T + K c_2 \log^{1+\delta}{T} + 2Kc_32^{\ell_0} + \frac{2Kc_3 ( 2KM+ 1)}{e-2} + W \\
    &\sim O(\log^{1+\delta} T + W),
    \end{split}
\end{equation}
where $W$ is the number of time units for which there was an adversarial attack on any arm. Note that $W$ can be replaced by the number of time units for which there was an adversarial attack on any arm during the exploration and exploitation phases only.

\section{Matching Phase}\label{sec:matching_phase}

Strategic form games in game theory are used to model situations where players choose actions simultaneously and each player has a utility function $u_j: \mathcal{A} \to [0,1]$, that assigns a real valued payoff (utility) to each action profile $\mathbf{a} \in \mathcal{A}$. In order to pose our multi-player MAB problem as a strategic form game, we need to design the utility functions of the players in a way such that the system regret is minimized, or equivalently, the system performance is maximized. For an action profile $\mathbf{a}$ and adversarial attack $\mathbf{w}$, the utility function received by each player $k$ is defined as:
\begin{equation}
    u_k(\mathbf{a}, \mathbf{w}) = 
    \begin{cases}
        \hat{\mu}_k(a_k,1) &\text{if} \;\; n(a_k) = 1 \text{    and } w_{a_k} = 0 \\
        0 &\text{if} \;\; n(a_k) \geq 2 \text{ 
   or  } \;\; w_{a_k} = 1.
    \end{cases}
\end{equation}
Note that each player $k$ does not know $n(a_k)$ or $w_{a_k}$, but instead only observes the reward instance $r_k(a_k,n(a_k),\mathbf{w})$. Thus, each player $k$ calculates their utility received as follows:
\begin{equation}
    u_k(\mathbf{a}, \mathbf{w}) = 
    \begin{cases}
        \hat{\mu}_k(a_k,1) &\text{if} \;\; r_k(a_k,n(a_k),\mathbf{w}) > 0 \\
        0 &\text{if} \;\; r_k(a_k,n(a_k),\mathbf{w}) = 0.
    \end{cases}
\end{equation}

A similar utility function is used in \cite{got}, where it is assumed that collisions result in zero rewards for the colliding players. Note that in the work in \cite{got} where there are no adversarial attacks, the utility for a player $k$ being $0$ can be only due to a collision from another player. However, in our setting, the utility being $0$ can be due to a collision from another player or an adversarial attack. Thus, we need to design the matching phase such that the players are not affected by the adversarial attack causing the utility to be $0$, as described in Algorithm \ref{alg:matching}. To this effect we use one-bit communication rounds in the matching phase. Since there are $O(\log T)$ matching phases in a time horizon of $T$, we utilize $O(\log^{1+\delta}T)$ one-bit communication rounds in total. The action profile that maximizes the sum of the utilities is called an efficient action profile. If the exploration phase is successful, then by our choice of the utility function, the efficient action profile maximizing the sum of the utilities is also the same as the optimal action profile maximizing the sum of expected rewards under $\mathbf{w^0}$. This is detailed in the Appendix. 

The matching phase lasts for $\tau_\ell$ time units, and $\epsilon, \beta \in (0,1)$ are parameters of the algorithm. Each player $k$ is associated with a state $Z_{t,k} = [\Bar{a}_{t,k}, \Bar{u}_{t,k}, S_{t,k}]$ during time $t$, where $\Bar{a}_{t,k} \in [M]$ is the baseline action of the player, $\Bar{u}_{t,k} \in [0,1]$ is the baseline utility of the player and $S_{t,k} \in \{C,D\}$ is the mood of the player ($C$ denotes \enquote{content} and $D$ denotes \enquote{discontent}). When the player is content, the baseline action is chosen with high probability ($1 - \epsilon^\kappa$) and every other action is chosen with uniform probability. The parameter $\epsilon$ is provided as an input to the algorithm. If the player is discontent, the action is chosen uniformly from all arms and there is a high probability that the player would choose an arm different from the baseline action. The goal in designing the matching phase algorithm is for all the players to align their baseline actions to the efficient action profile and be content in this state. Note that the action dynamics do not depend on the utilities. We modify the state update step from \cite{marden} as follows. If the player is content and their baseline action matches the action played, the state remains the same. Note that in the algorithm in \cite{marden} (and subsequently in \cite{got}), both the baseline action and baseline utilities are compared to the action played and utility received when the player is content. However, in our setting, the utility can be changed by an adversary. Thus, we force a content player to stay content when they play their baseline action. In the case that a content player explores with probability $\epsilon^\kappa$ or if the player is discontent, the estimated state for the next play is chosen stochastically based on the utility. The rationale behind the particular probabilities chosen is that when the utility received is high, the player is more likely to be content. To compensate for the change in the state update step where we force a content player playing their baseline action to remain content to combat adversarial attacks, we introduce a one-bit communication round for the players to synchronize their moods. If a player $k$ is content and the rest of the players are content as well, player $k$ remains content. On the other hand, if a player $k$ is content and there is any player in the remaining $K-1$ players who is discontent, the player updates their state stochastically as in \eqref{eq:update_mood}. A discontent player remains discontent. During the matching phase algorithm, each player keeps a count of the number of times each arm was played that resulted in the player being content.
The action chosen by the player for the exploitation phase is the arm played most frequently from epochs $\lceil \frac{\ell}{2}\rceil$ to $\ell$ that resulted in the player being content.

In order to analyze the average regret, we provide guarantees on the estimation of the efficient action profile in the matching phase of our algorithm. The analysis relies on the theory of regular perturbed Markov decision processes \cite{young}. Using the theory of perturbed Markov chains in \cite{young}, we characterize the stochastically stable states (see \cite{marden}) for our setting with adversarial attacks and prove that the stochastically stable state that maximizes the sum of utilities is played for the majority of time in the matching phase with high probability. More specifically, the state with the baseline actions and utilities corresponding to the optimal action profile and all players being content is the stochastically stable state. 
In the following lemma, we bound the probability of the optimal action profile not being played during the exploitation phase of epoch $\ell$ (identified from the matching phase of epochs $\lceil\frac{\ell}{2}\rceil$ to $\ell$), given that event $E^\ell$ does not occur (\textit{i.e.}, the exploration phases of epochs $\lceil\frac{\ell}{2}\rceil$ to $\ell$ were successful).

    \begin{lemma}\label{lem:matching_phase_success_prob}
In some epoch $\ell$, let 
$\mathbf{a}^*=\mathop{\arg\max}\limits_{\mathbf{a}, \mathbf{w^0} } \sum_{k=1}^K u_k(\mathbf{a}, \mathbf{w^0}).$
and let $\mathbf{a'} = [a'_1,...,a'_K]$ where 
$a'_k$ is the action profile played in the exploitation phase of epoch $\ell$ by player $k$. Assume that for all players $k \in [K]$ and for all arms $m \in [M]$, the exploration phase is successful. Then
%Then for small enough $\epsilon$, 
\[
P\{\mathbf{a^*} \neq \mathbf{a'}\} \leq \left(C_0 \exp{(- C_\rho \ell^{\delta})}\right)^\ell
\]
for some $C_0, C_\rho > 0$.
\end{lemma} 

A detailed discussion on the analysis of the matching phase algorithm and the proof of the above lemma is provided in the Appendix.

\section{Experimental Results}

In this section, we present some illustrative experimental results. We consider the case with $K = M = 3$. The mean rewards for players $1,2$ and $3$ for the arms $1,2$ and $3$ are $[0.8, 0.6, 0.4]$, $[0.2, 0.7, 0.3]$ and $[0.1, 0.7, 0.5]$, respectively. The rewards are drawn from beta distributions with the corresponding means. Thus, the optimal action profile is $[1,2,3]$ and the considered system has $\Delta = 0.03$. We set $\delta = 0$ in the experiments. We set $T_0 = c_2 = 2000$ and $c_3 = 10000$ and $\epsilon = 10^{-4}, \beta = 2, \kappa = 3$. The algorithm was run for $10$ epochs and the experiment was repeated for $100$ iterations and the accumulated regret averaged over the iterations. The adversary attacks one arm at random at each time step during the exploration and matching phases of all the epochs with probability $0.4$. 

From Figure \ref{one}, we see that the average accumulated regret grows sub-linearly with time. We can also observe that the average regret  incurred during the exploitation phase of each epoch is small, as the matching phase converges to the optimal action profile with high probability despite the presence of adversarial attacks.

\begin{figure}[t]
\begin{center}
\centerline{\includegraphics[scale = 0.8]{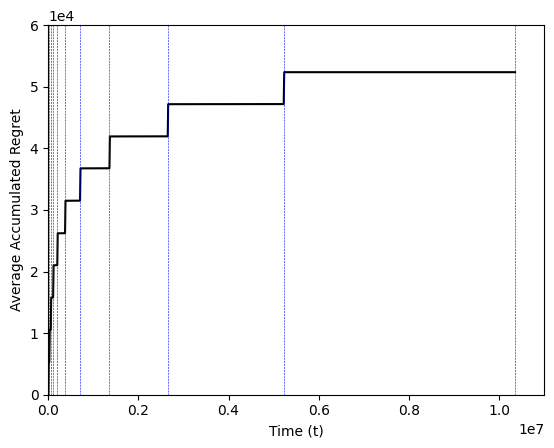}}
\caption{Average accumulated regret as a function of time}
\label{one}
\end{center}
\vskip -0.3in
\end{figure}

\appendices

\section{Exploration Phase}\label{apdx:exploration}

The algorithm for the exploration phase is detailed in Algorithm \ref{alg:exploration}.

\begin{algorithm}
   \caption{Exploration phase for player $k$ for epoch $\ell$}
   \label{alg:exploration}
\begin{algorithmic}
    \FOR{ $m = 1$ to $M$}
    \STATE \textbullet~Player $k$ pulls arm $(k + m - 1) \mod M$ until $T_0 \ell^\delta$ non-zero reward observations (i.e., without any adversarial attack) are recorded on that arm
    \STATE \textbullet~Player $k$ sends a one-bit message of $1$ to all the players to communicate their success in the exploration phase 
    \STATE \textbullet~Player $k$ remains on arm $(k + m - 1) \mod M$ until they receive a 1 from all the remaining players 
    \STATE \textbullet~Estimated mean reward $\hat{\mu}_{k}(m,1)$ of arm $(k + m - 1) \mod M$  is updated with the average of the non-zero rewards observed
    \ENDFOR
\end{algorithmic}
\end{algorithm}

The exploration phase is for player $k$ to obtain estimates of the mean rewards of arms $m \in [M]$. During the exploitation phase of epoch $\ell$, each player plays the most frequently played action during the matching phases of epochs $\lceil \frac{\ell}{2} \rceil$ to $\ell$ while being in a content state. Thus, we bound the probability that the estimated mean rewards deviate from the mean rewards by more than $\Delta$ in at least one of the epochs from $\lceil \frac{\ell}{2} \rceil$ to $\ell$, by $O(e^{-\ell})$.

\begin{lemma}\label{lem:exploration}
Given $\Delta$ as defined in Section \ref{sec:system_model}, for a fixed player $k$, arm $m$, let $E_{k,m}^\ell$ denote the event that there exists at least one epoch $i$ with $\lceil \frac{\ell}{2} \rceil \leq i \leq \ell$ such that the estimate of the mean reward $\hat{\mu}_k(m,1)$ obtained after the exploration phase of epoch $i$ satisfies $|\hat{\mu}_k(m,1) - \mu_{k}(m,1)| \geq \Delta$. Then  
\begin{equation}
    P(E_{k,m}^\ell) \leq  \frac{e^{- \frac{T_0 \Delta^2}{2} (\frac{\ell}{4})^{\delta} \ell}}{1 - e^{-T_0 \Delta^2 (\frac{\ell}{4})^{\delta}}}.
\end{equation}
\end{lemma}

\begin{proof}
    We want to bound the probability of the event $E_{k,m}^\ell$ that after the exploration phase of epoch $\ell$, there exists at least epoch $\lceil \frac{\ell}{2} \rceil \leq i \leq \ell $ such that the estimated mean reward after the exploration phase of epoch $i$ $\hat{\mu}_k(m,1)$ satisfies $|\mu_k(m,1) - \hat{\mu}_k(m,1)| \geq \Delta$.

For a fixed player $k$, arm $m$, the number of total number of reward samples obtained after the exploration phase of epoch $i$ from the corresponding reward distribution with mean reward $\mu_k(m,1)$ is $T_0 \sum_{j = 1}^i j^\delta \geq T_0 (\frac{i}{2})^{1 + \delta}$. After the exploration phase of epoch $i$, the probability that the estimated mean reward $\hat{\mu}_k(m,1)$ deviates from $\mu_k(m,1)$ by more than $\Delta$ is:

\begin{align*}
    &P\left\{\text{After epoch $i$ :} \; |\hat{\mu}_k(m,1) - \mu_k(m,1)| \geq \Delta \right\} \\
    &\leq e^{-2T_0\Delta^2 \sum_{j = 1}^i j^\delta} \\
    &\leq e^{-2T_0\Delta^2 (\frac{i}{2})^{1 + \delta}}
\end{align*}

It follows that
\begin{equation}
    \begin{split}
        &P\left(E_{k,m}^\ell \right) \\
    &= P \left(\bigcup_{i = \lceil \frac{\ell}{2} \rceil}^{\ell} \left\{\text{After epoch $i$ :} \; |\hat{\mu}_k(m,1) - \mu_k(m,1)| \geq \Delta \right\} \right) \\
    &\leq \sum_{i = \lceil \frac{\ell}{2} \rceil}^{\ell} P\left\{\text{After epoch $i$ :} \; |\hat{\mu}_k(m,1) - \mu_k(m,1)| \geq \Delta \right\} \\
    &\leq \sum_{i = \lceil \frac{\ell}{2} \rceil}^{\ell} e^{-2T_0\Delta^2 (\frac{i}{2})^{1 + \delta}} \\
    &\leq \sum_{i = \lceil \frac{\ell}{2} \rceil}^{\ell} e^{-T_0\Delta^2 (\frac{\ell}{4})^{\delta}i} \leq \frac{e^{- \frac{T_0 \Delta^2}{2} (\frac{\ell}{4})^{\delta} \ell}}{1 - e^{-T_0 \Delta^2 (\frac{\ell}{4})^{\delta}}}.
    \end{split}
\end{equation}

\end{proof}

\section{Proof of main result}

    Let $L_T$ be the last epoch within a  time-horizon of  $T$. The regret incurred during the $L_T$ epochs can be analyzed as the sum of the regret incurred during the three phases of the algorithm. The exploitation phase in epoch $\ell$ of the algorithm lasts for $c_3 2^{\ell}$ time units. Thus, it is easy to see that $L_T < \log{T}$. Let $R_1$, $R_2$ and $R_3$ denote the regret incurred  over $L_T$ epochs, during the exploration phase, the matching phase, and the exploitation phase, respectively. Let $\ell_0$ be the first epoch $\ell$ such that 
\begin{align}
    \frac{T_0 \Delta^2}{2} {\left(\frac{\ell}{4}\right)}^{\delta} & \geq 1, \label{exp1} \\
    C_0 e^{-C_\rho \ell^{\delta/2}} &\leq e^{-1} \label{match1}
\end{align}
where $C_\rho$ is defined in \eqref{crho} in Appendix~\ref{apdx:proof_of_lemmas}. Note that \eqref{exp1} - \eqref{match1} hold for all $\ell \geq \ell_0$. We upper bound the regret incurred during the exploration and matching phases by $K$ (which is the maximum regret that can be accumulated in one time unit) times the total time taken by these phases. For the exploitation phase, we incur regret only when the action profile played during this phase is not optimal.

\begin{enumerate}
    \item Exploration phase: Since the exploration phase in each epoch $\ell$ proceeds for at most $MT_0 \ell^\delta + W^\ell_\mathrm{exp}$ time units,
    \begin{equation}
        R_1 \leq \sum_{\ell = 1}^{L_T} K (M T_0 \ell^\delta + W^\ell_\mathrm{exp}) \leq K M T_0 \log^{1 + \delta} T + K W^\ell_\mathrm{exp}.
    \end{equation}
    \item Matching phase: In epoch $\ell$, the matching phase runs for $\tau_\ell = c_2  \ell^{\delta}$ time units. Thus 
    \begin{equation}
        \begin{split}
            R_2 &\leq  K \sum_{\ell = 1}^{L_T} c_2  \ell^{\delta} \leq K c_2  L_T^{1 + \delta} \leq K c_2  \log^{1+\delta}{T}.
        \end{split}
    \end{equation}
    \item Exploitation phase: In the exploitation phase, regret is incurred in the following three events:
    \begin{enumerate}
        \item Let $E^\ell$ denote the event that there exists some player $k \in [K]$, arm $m \in [M]$, such that there exists some epoch $i$, with $\lceil \frac{\ell}{2} \rceil \leq i \leq \ell$, such that the estimate of the mean reward $\hat{\mu}_k(m,1)$ obtained after the exploration phase of epoch $i$ satisfies $|\hat{\mu}_k(m,1) - \mu_{k}(m,1)| \geq \Delta$.  
        %Let the probability of this event be $P(E^\ell)$.
        \item Let $F^\ell$ denote the event that given that all the players have $|\mu_{k}(m,1) - \hat{\mu}_{k}(m,1)| \leq \Delta$ for all $m \in [M]$ and all $n \in [N]$ for all epochs $\lceil \frac{\ell}{2} \rceil $ to $ \ell$, the action profile chosen in the matching phase of epoch $\ell$ is not optimal. 
        %Let the probability of this event be $P(F^\ell)$.
        \item There is an adversarial attack on any arm.
    \end{enumerate}
    From Lemma \ref{lem:exploration}, we have an upper bound on the probability of event $E_{k,m}^\ell$ that for some fixed player $k$, arm $m$, there exists some epoch $i$, with $\lceil \frac{\ell}{2} \rceil \leq i \leq \ell$, such that the estimate of the mean reward $\hat{\mu}_k(m,1)$ obtained after the exploration phase of epoch $i$ satisfies $|\hat{\mu}_k(m,1) - \mu_{k}(m,1)| \geq \Delta$. Thus, we have that
    \begin{align}
        P(E^\ell) &= P\left( \underset{k \in [K], m\in [M]}{\bigcup} E_{k,m}^\ell \right) \\
        &\leq KM P(E_{k,m}^\ell) \\
        &\leq \frac{KMe^{- \frac{T_0 \Delta^2}{2} (\frac{\ell}{4})^{\delta} \ell}}{1 - e^{-T_0 \Delta^2 (\frac{\ell}{4})^{\delta}}}.\label{exp_prob}
    \end{align}
   We also have the following upper bound on the probability of event $F^\ell$ from Lemma \ref{lem:matching_phase_success_prob}:
    \begin{equation}\label{match_prob}
        P(F^\ell) \leq \left(C_0 \exp{(-C_\rho\ell^{\delta})}\right)^\ell.
    \end{equation} 
 Using \eqref{exp1} and \eqref{match1} and the upper bounds in \eqref{exp_prob} and \eqref{match_prob}, we have that for all epochs $\ell \geq \ell_0$:
 \begin{align}
     P(E^\ell) &\leq 2KMe^{-\ell} \label{exp_uppbd}\\
     P(F^\ell) &\leq  e^{-\ell} \label{match_uppbd}.
 \end{align}
    Therefore,
    \begin{align}
            R_3 &= K\sum_{\ell = 1}^{L_T} c_3 2^\ell (P(E^\ell) + P(F^\ell)) + KW^\ell_\mathrm{exploit} \label{exploit_reg}\\
            &\leq 2Kc_32^{\ell_0} + K c_3 \sum_{\ell = \ell_0}^{L_T}  (2KM + 1) \left(\frac{2}{e}\right)^\ell + KW^\ell_\mathrm{exploit}     \label{upp_bound} \\
            &\leq 2Kc_32^{\ell_0} + \frac{2Kc_3 ( 2KM+ 1)}{e-2} + KW^\ell_\mathrm{exploit},
    \end{align}
    where $KW^\ell_\mathrm{exploit}$ is the number of time units in epoch $\ell$ where there was an adversarial attack on at least one arm . Note that \eqref{exploit_reg} follows from the fact that regret is incurred in the exploitation phase only if events $E^\ell$ or $F^\ell$ occur or if there is an adversarial attack, \eqref{upp_bound} follows from upper bounding the regret incurred during the first $\ell_0$ epochs by the maximum possible regret, and using \eqref{exp_uppbd} and \eqref{match_uppbd} to upper bound the regret for epochs greater than $\ell_0$. 
\end{enumerate}

    Thus 
\begin{equation}
    \begin{split}
    R(T) &= R_1 + R_2 +R_3 \\
    &\leq K M  T_0 \log^{1 + \delta} T + K c_2 \log^{1+\delta}{T} \\ 
    &+ 2Kc_32^{\ell_0} + \frac{2Kc_3 ( 2KM+ 1)}{e-2} + W \\
    &\sim O(\log^{1+\delta} T + W),
    \end{split}
\end{equation}
where $W$ is the number of time units for which there was an adversarial attack on any arm. Note that $W$ can be replaced by the number of time units for which there was an adversarial attack on any arm during the exploration and exploitation phases only.

\section{Matching Phase}\label{apdx:matching_phase}

Strategic form games in game theory are used to model situations where players choose actions simultaneously (rather than sequentially) and do not have knowledge of the actions of other players. In such games, each player has a utility function $u_j: \mathcal{A} \to [0,1]$, that assigns a real valued payoff (utility) to each action profile $\mathbf{a} \in \mathcal{A}$. An algorithm that works under the assumption that every agent can observe only their own action and utility received is called a payoff based method. The matching phase of our proposed algorithm builds on \cite{marden}, where a payoff based decentralized algorithm that leads to maximizing the sum of the utilities of the players is presented.  Since the utility function in our setting would depend on both the action profile of the players and the adversarial attacks, we introduce one bit communication rounds to combat the effect of the adversaries. In order to pose out multi-player MAB problem as a strategic form game, we need to design the utility functions of the players in a way such that the system regret is minimized, or equivalently, the system performance is maximized. For an action profile $\mathbf{a}$ and adversarial attack $\mathbf{w}$, the utility function received by each player $k$ is defined as:
\begin{equation}
    u_k(\mathbf{a}, \mathbf{w}) = 
    \begin{cases}
        \hat{\mu}_k(a_k,1) &\text{if} \;\; n(a_k) = 1 \text{    and } w_{a_k} = 0 \\
        0 &\text{if} \;\; n(a_k) \geq 2 \text{ 
   or  } \;\; w_{a_k} = 1.
    \end{cases}
\end{equation}
Note that each player $k$ does not know $n(a_k)$ or $w_{a_k}$, but instead only observes the reward instance $r_k(a_k,n(a_k),\mathbf{w})$. Thus, each player $k$ calculates their utility received as follows:
\begin{equation}
    u_k(\mathbf{a}, \mathbf{w}) = 
    \begin{cases}
        \hat{\mu}_k(a_k,1) &\text{if} \;\; r_k(a_k,n(a_k),\mathbf{w}) > 0 \\
        0 &\text{if} \;\; r_k(a_k,n(a_k),\mathbf{w}) = 0.
    \end{cases}
\end{equation}

A similar utility function is used in \cite{got}, where it is assumed that collisions result in zero rewards for the colliding players. Note that in the work in \cite{got} where there are no adversarial attacks, when players receive zero rewards on collisions, $\hat{\mu}_k(a_k, n(a_k)) = 0$ whenever $n(a_k) \geq 2$. Therefore, in the work of \cite{got}, the utility for a player $k$ being $0$ can be only due to a collision from another player. 
%On the other hand, in \cite{magesh2021decentralized}, since player $k$ does not know $n(a_k)$, the utility is also not known exactly. Each player needs to {estimate} $n(a_k)$ based on the instantaneous rewards seen in the matching phase, and use this to estimate their utility. However, the instantaneous rewards can be modified 

However, in our setting with adversarial attacks and zero rewards on collisions, the utility being $0$ can be due to a collision from another player or an adversarial attack. Thus, we need to design the matching phase such that the players are not affected by the adversarial attack causing the utility to be $0$ as described in Algorithm \ref{alg:matching}.

The action profile that maximizes the sum of the utilities is called an efficient action profile. The following lemma states that if for all $k\in [K]$ and $m \in [M]$,  $|\hat{\mu}_k(m,1) - \mu_{k}(m,1)| \leq \Delta$, then by our choice of the utility function, the efficient action profile maximizing the sum of the utilities is also the same as the optimal action profile maximizing the sum of expected rewards under $\mathbf{w^0}$. 

\begin{lemma}
    If for all $k\in [K], m \in [M]$, the following condition is satisfied : 
    \begin{equation}\label{err}
        |\hat{\mu}_k(m,1) - \mu_{k}(m,1)| \leq \Delta,
    \end{equation} 
    then 
    \begin{equation}
        \mathop{\arg\max}_{\mathbf{a} \in \mathcal{A}} \sum_{k=1}^{K} \mu_k(a_k,1) = \mathop{\arg\max}_{\mathbf{a} \in \mathcal{A}} \sum_{j=1}^{K} \hat{\mu}_k(a_k,1).
    \end{equation}
\end{lemma}

The condition given by (\ref{err}) is guaranteed with high probability by the exploration phase of the algorithm. The proof of the above lemma is similar to the proof of \cite[Lemma 1]{got} and \cite[Lemma 2]{magesh2021decentralized}. Thus, the efficient action profile that maximizes the sum of utilities (estimated mean rewards) is the same as the optimal action profile that minimizes regret or equivalently, maximizes system performance.

\subsection{Description of the Matching Phase Algorithm}

This phase lasts for $\tau_\ell$ time units, and $\epsilon, \beta \in (0,1)$ is a parameter of the algorithm. Each player $k$ is associated with a state $Z_{t,k} = [\Bar{a}_{t,k}, \Bar{u}_{t,k}, S_{t,k}]$ during time $t$, where $\Bar{a}_{t,k} \in [M]$ is the baseline action of the player, $\Bar{u}_{t,k} \in [0,1]$ is the baseline utility of the player and $S_{t,k} \in \{C,D\}$ is the mood of the player ($C$ denotes \enquote{content} and $D$ denotes \enquote{discontent}). Note that $\{Z_{1,j}, Z_{2,j}, ...\}$ are the states resulting from running the matching phase algorithm (essentially the state update and the synchronisation steps) with the utilities $u_k(\mathbf{a}, \mathbf{w})$ which depends on both the action profile of all the players and the adversarial attacks.

When the player is content, the baseline action is chosen with high probability ($1 - \epsilon^\kappa$) and every other action is chosen with uniform probability. The parameter $\epsilon$ is provided as an input to the algorithm. If the player is discontent, the action is chosen uniformly from all arms and there is a high probability that the player would choose an arm different from the baseline action. This part of the algorithm constitutes the action dynamics. The baseline action can be interpreted as the arm the agent expects to play for a long time eventually and the baseline utility can be interpreted as the payoff the player expects to receive upon playing the baseline action when there are no adversarial attacks. The player being content is an indication that the payoff received by the player while playing his baseline action in the absence of adversarial attacks has been satisfactory. Thus, the goal in designing the matching phase algorithm is for all the players to align their baseline actions to the efficient action profile and be content in this state. Note that the action dynamics do not depend on the utilities.

We have seen the justification for using the utility function
\begin{equation}
    u_k(\mathbf{a}, \mathbf{w}) = 
    \begin{cases}
        \hat{\mu}_k(a_k,1) &\text{if} \;\; r_k(a_k,n(a_k),\mathbf{w}) > 0 \\
        0 &\text{if} \;\; r_k(a_k,n(a_k),\mathbf{w}) = 0.
    \end{cases}
\end{equation} in the introduction of Section \ref{sec:matching_phase}. Note that when the utility of the player $k$ is $0$, it could be due to one or both of the following reasons:
\begin{enumerate}[(i)]
    \item There are other players colliding on the arm $a_k$.
    \item There is an adversarial attack on arm $a_k$, i.e., $w_{a_k} = 0$.
\end{enumerate}
Thus, the adversarial attack needs to be accounted for when the state is updated, since the player cannot distinguish between a collision from another player and an adversarial attack. 

We modify the state update step from \cite{marden} as follows. The player updates their current state by comparing the action played and the utility received with the baseline action and baseline utility associated with the current state. If the player is content and their baseline action matches the action played, the state remains the same. Note that in the algorithm in \cite{marden} (and subsequently in \cite{got}), both the baseline action and baseline utilities are compared to the action played and utility received when the player is content. However, in our setting, the utility can be changed by an adversary. For instance, in the case where a content player $k$ plays their baseline action $\Bar{a}_k$, the utility received by the player could be $0$ instead of $\hat{\mu}_k(\Bar{a}_k,1)$ when there are no other players colliding on that arm. In such a scenario, comparing the utility and the baseline utility could result in the player becoming discontent, which could be undesirable. Thus, we force a content player to stay content when they play their baseline action. In the case that a content player explores with probability $\epsilon^\kappa$ or if the player is discontent, the estimated state for the next play is chosen stochastically based on the utility. The rationale behind the particular probabilities chosen is that when the utility received is high, the player is more likely to be content. To compensate for the change in the state update step where we force a content player playing their baseline action to remain content to combat adversarial attacks, we introduce a one-bit communication round for the players to synchronise their moods. Each player receives the mood of all the players in the system. If a player $k$ is content and the rest of the players are content as well, player $k$ remains content. On the other hand, if a player $k$ is content and there is any player in the remaining $K-1$ players who is discontent, the player updates their state stochastically as in \eqref{eq:update_mood}. A discontent player remains discontent. Further details on the necessity of this step are provided in the following subsection.

The utility each player receives is equivalent to feedback from the system on how the entire action profile affects the reward received by this player. If the player receives a lower payoff due to that arm not being good or due to collisions, there is a higher probability of the player becoming discontent and exploring other arms. On the other hand, if the payoff received is higher, there is a higher probability of the player staying content and exploiting the same arm again. Thus the agent dynamics and state dynamics balance the exploration-exploitation trade-off in the multi-player MAB setting. 

During the matching phase algorithm, each player keeps a count of the number of times each arm was played that resulted in the player being content:
$$W^\ell(k,m) = \sum_{t = 1}^{\tau_\ell} \identityf{(a_{t,k} = m, S_{t,k} = C)}.$$
The action chosen by the player for the exploitation phase is the arm played most frequently from epochs $\lceil \frac{\ell}{2}\rceil$ to $\ell$ that resulted in the player being content: 
   $$a_k = \mathop{\arg\max}_m \sum_{i = \lceil \frac{\ell}{2}\rceil}^\ell W^i(k,m).$$

\subsection{Analysis of the Matching Phase Algorithm}

The matching phase algorithm is inspired from the work in \cite{marden}, and the guarantees provided there state that the action profile maximizing the sum of the utilities of the players (efficient action profile) is played for a majority of the time. In this subsection, we provide guarantees on the estimation of the efficient action profile in the matching phase of our algorithm. The analysis relies on the theory of regular perturbed Markov decision processes \cite{young}. 

The dynamics of the matching phase algorithm induce a Markov chain over the state space $\mathcal{Z} = \Pi_{j=1}^K ([M] \times [0,1] \times \mathcal{M})$ where $\mathcal{M} = \{C,D\}$, $\textit{i.e.}$, each state $z = [z_1, \ldots, z_K] \in \mathcal{Z}$ is a vector of the states of all players, where $z_k = (\Bar{a}_k, \Bar{u}_k, S_k)$. Let $P^0$ denote the probability transition matrix of the process when $\epsilon = 0$ and $P^\epsilon$ denote the transition matrix when $\epsilon > 0$. The process $P^\epsilon$ is a regular perturbed Markov process if for any $z,z' \in \mathcal{Z}$ (Equations (6),(7) and (8) of Appendix of \cite{young}): 
\begin{enumerate}\label{conditions}
    \item $P^\epsilon$ is ergodic
    \item $\lim_{\epsilon \to 0} P^\epsilon_{zz'} = P^0_{zz'} $
    \item $P^\epsilon_{zz'} > 0$ implies for some $\epsilon$, there exists $r(z,z') \geq 0$ such that $$0< \lim_{\epsilon \to 0} \epsilon^{-r(z,z')} P^\epsilon_{zz'} <\infty$$
\end{enumerate}
The process $P^0$ is referred to as the unperturbed process. The value of $r(z,z')$ satisfying the third condition is called the resistance of the transition $z \to z'$, denoted by $r(z,z')$. The transitions of resistance $0$ are the same ones that are feasible under $P^0$.

Let $\mu^\epsilon$ be the unique stationary distribution of $P^\epsilon$, where $P^\epsilon$ is a regular perturbed Markov process. Then $\lim_{\epsilon \to 0} \mu^\epsilon$ exists and the limiting distribution $\mu^0$ is a stationary distribution of $P^0$. The stochastically stable states are the support of $\mu^0$. The main result of \cite[Theorem 3.2]{marden} states that the stochastically stable states of the Markov chain induced by their proposed payoff based decentralized learning rule maximize the sum of the utilities of the players. However,  \cite[Theorem 3.2]{marden} cannot be applied directly in our setting, because:
\begin{enumerate}
    \item The state update step in \eqref{eq:content_update} is different from that in \cite{marden}. Irrespective of the utility received, a content player playing their baseline action will remain in the same state in \eqref{eq:content_update}. This is done to combat any adversarial attacks mimicking a collision from another player to change the utility received by a player, and causing the player to enter a discontent state.
    \item The work in \cite{marden} assumes that the structure of the utility functions and the game induced by their algorithm satisfies a property called interdependence \cite[Definition 1]{marden}. The interdependence property implies that it is not possible to divide the agents into two distinct subsets, where the actions of agents in one subset do not affect the utilities of those in the other. However, the change in the state update step in \eqref{eq:content_update} in our proposed method results in our game violating this interdependence property. Thus, in order to mimic the effect of the interdependence property, we employ one bit communications to synchronise the moods of the players as proposed in the last step of Algorithm \ref{alg:matching}, inspired by the work in \cite{baras_one_bit_comm}. 
\end{enumerate}

Thus, using the theory of perturbed Markov chains in \cite{young}, we characterize the stochastically stable states for our setting with adversarial attacks and prove that the stochastically stable state that maximizes the sum of utilities is played for the majority of time in the matching phase with high probability. Our proof adapts some techniques from the work in \cite{ramesh2016distributed}, where they consider a distributed learning setting with disturbances. However, in contrast to the work in \cite{ramesh2016distributed}, our method does not use the knowledge of any system parameters describing the effect of the adversarial disturbances (see \cite[Equation 3]{ramesh2016distributed}), and we do not place any constraints on the utility functions (see \cite[Definition 2.1]{ramesh2016distributed}).

\begin{lemma}\label{lem:perturbation}
    The Markov process with the transition $P^\epsilon$ is a regular perturbation of $P^0$.
\end{lemma}

The proof is provided in Appendix \ref{apdx:proof_of_lemmas}. 

Consider the following classes of states:
\begin{itemize}
        \item Let the set of states for which all players are discontent be:
        \begin{equation}
            \mathcal{Z}_D = \left\{ z \in \mathcal{Z} : S_k = D, \;\; \forall \;\; k \in [K] \right\}.
        \end{equation}
        \item Consider a state in which all agents are content. Depending on whether the agents' baseline utilities align with the baseline action for some adversarial attack, such states are constructed in an iterative manner as follows. Consider any partition of the set of  agents $[K]$ into $i$ disjoint subsets $J_1, \ldots, J_i$. For $i = 1$, that would just be the set of all agents $J_1 = [K]$. For $i = K$, this would be $J_1, \ldots, J_K$, where each $J_i$ is a singleton set with a unique agent. Starting with $i = 1$,i
        \begin{equation}
        \begin{split}
            \mathcal{Z}_{C,1} = &\{ [z_1, \ldots, z_K] \in \mathcal{Z} : \\
 &z_{j_1} = (\Bar{a}_{j_1}, u_{j_1}(\mathbf{\Bar{a}}, \mathbf{w}), C) \text{ for some } \mathbf{w},  \;\; \forall \;\; j_1 \in J_1  \}.
        \end{split}
        \end{equation}
        Thus, for every action profile $\mathbf{\Bar{a}}$ and adversarial attack $\mathbf{w}$, there exists a state in $\mathcal{Z}_{C,1}$ with the baseline utilities aligned to the baseline actions for that adversarial attack. In the subsequent recurrence classes, the baseline utilities of all players need not be aligned with their baseline actions. Consider $i = 2$ with two subsets $J_1$ and $J_2$. A state where players in $J_2$ in state $\mathcal{Z}_{C,1}$ explore different arms and become content, while the players in $J_1$ remain content since they do not explore, and the baseline utilities of some players in $J_1$ are not aligned with the new action profile, can be defined as:
        \begin{equation}
        \begin{split}
            \mathcal{Z}_{C,2} = &\{ [z_1, \ldots, z_K] \in \mathcal{Z} : \\
            &z_{j_2} = (\Bar{a}_{j_2}, u_{j_2}(\mathbf{\Bar{a}}, \mathbf{w}), C) \text{ for some } \mathbf{w},  \;\; \forall \;\; j_2 \in J_2; \\
            & \exists j \in J_1 \text{ s.t. } \Bar{u}_{j} \neq (\Bar{a}_{j},  u_{j}(\mathbf{\Bar{a}}, \mathbf{w}), C) \text{ for any } \mathbf{w}, \\
            &z_{j_1} = z'_{j_1} \text{ for some } z' \in \mathcal{Z}_{C,1},  \;\; \forall \;\; j_1 \in J_1 \}.
        \end{split}
        \end{equation}
        This can be extended to $i = k$ for $k >2$ similarly, where a subset of agents in $\mathcal{Z}_{C,k-1}$ explore and become content with their new utilities while the baseline utilities of some players in $J_1 \cup \ldots \cup J_{k-1}$ are not aligned with the new action profile. These classes can be described as follows:
        \begin{equation}
        \begin{split}
            \mathcal{Z}_{C,k} = &\{ [z_1, \ldots, z_K] \in \mathcal{Z} : \\
            &z_{j_k} = (\Bar{a}_{j_k}, u_{j_k}(\mathbf{\Bar{a}}, \mathbf{w}), C) \text{ for some } \mathbf{w},  \;\; \forall \;\; j_k \in J_k; \\
            & \exists j \in J_1 \cup \ldots \cup J_{k-1} \text{ s.t. } \Bar{u}_{j} \neq (\Bar{a}_{j},  u_{j}(\mathbf{\Bar{a}}, \mathbf{w}), C) \text{ for any } \mathbf{w}, \\
            &z_{j_i} = z'_{j_i} \text{ for some } z' \in \mathcal{Z}_{C,k-1},  \;\; \forall \;\; j_i \in J_1 \cup \ldots \cup J_{k-1} \}.
        \end{split}
        \end{equation}
    \end{itemize}

\begin{lemma}\label{lem:recurrent_states}
    The recurrence classes of the game described in Algorithm \ref{alg:matching} for $\epsilon = 0$ are:
    \begin{itemize}
        \item the set $\mathcal{Z}_D$,
        \item the singletons in $\mathcal{Z}_{C,i}$ for $i = 1, \ldots, K$. 
    \end{itemize}
\end{lemma}

The proof is provided in Appendix \ref{apdx:proof_of_lemmas}. 

Note that the change in the state update step from the work in \cite{marden} is necessary due to the adversarial attacks. If adversarial attacks could modify the utilities in the game proposed in \cite{marden}, the players in state $\mathcal{Z}_{C,1}$ would become discontent in the next time step under $P^0$, and thus the states in $\mathcal{Z}_{C,1}$ would not be recurrent classes. And the inclusion of the mood synchronisation step is necessary to make sure that any state with a mixture of content and discontent agents is not recurrent. 

In order to characterise the stochastically stable states, we will use the framework of resistance trees from \cite{young}. In the following lemmas, we characterise the resistances for the state transitions among the recurrence classes (from property (iii) of perturbed Markov chains), and characterise the state with the minimum stochastic potential (see Appendix \ref{apdx:mabs_preliminaries}). This state is the stochastically stable state.

\begin{lemma}\label{lem:resistances}
    The resistances for the state transitions among the recurrence classes  $\mathcal{Z}_D$ and singletons in $\mathcal{Z}_{C,i}$  are as follows. Let $z_D \in \mathcal{Z}_D$ and $z_{C,i} \in \mathcal{Z}_{C,i}$. 
    \begin{itemize}
        \item $r(z_D,z_{C,1}) = \sum\limits_{k=1}^K ({1 - u_{k}(\mathbf{a}, \mathbf{w})})$
        \item $r(z_D, z_{C,i}) \geq \min r(z_D,z_{C,1}) + \kappa i$ for $i \geq 2$
        \item $r(z_{C,1}, z_{C,1}) \in [\kappa,2\kappa] $
        \item $r(z_{C,1}, z_D) = \kappa $
        \item $r(z_{C,1}, z_{C,i}) \geq \kappa i$ for $i \geq 2$
        \item $r(z_{C,i}, z_{D}) = \kappa$ for $i \geq 2$
        \item $r(z_{C,i}, z_{C,1}) \geq \kappa i$ for $i \geq 2$
        \item $r(z_{C,i}, z_{C,i}) \geq \kappa  $ for $i \geq 2$
        \item $r(z_{C,i}, z_{C,j}) \geq \kappa |i-j| $ for $i,j \geq 2$ and $i \neq j$.
    \end{itemize}
\end{lemma}

\begin{lemma}\label{lem:not_stochastic_stable}
    The states in $\mathcal{Z}_D$ and $\mathcal{Z}_{C,i}$ for $i \geq 2$ are not stochastically stable states.
\end{lemma}

\begin{lemma}\label{lem:stochastic_potential}
    The stochastic potential of some state in $Z_{C,1}$ for a baseline action profile $\mathbf{a}$ and a adversarial attack $\mathbf{w}$ corresponding to that state is:
    \begin{equation}\label{eq:stochastic_potential}
       \gamma({z}_{C,1}) = \kappa \left( \sum\limits_{i = 1}^K |\mathcal{Z}_{C,i}| - 1\right) + \sum\limits_{k=1}^K (1 - u_k(\mathbf{a}, \mathbf{w})) 
    \end{equation}
\end{lemma}

The proofs of the above lemmas are provided in Appendix \ref{apdx:proof_of_lemmas}.

\begin{theorem}\label{thm:main_mab}
Under the dynamics defined in Algorithm \ref{alg:matching}, a state $z_{C,1} \in \mathcal{Z}_{C,1}$ is a stochastically stable state if and only if the action profile given by the baseline actions of all the players in this state maximizes the sum of their utilities across all adversarial attacks, the baseline utilities are aligned with the action profile and all the players are content. In other words, the action profile corresponding to the stochastically stable state maximizes the system welfare:
\begin{equation}\label{eq:max_sys_welfare}
    (\mathbf{\Bar{a}}, \mathbf{w^0}) \in \mathop{\arg\min}\limits_{\mathbf{a}, \mathbf{w}} \sum\limits_{k=1}^K (1 - u_k(\mathbf{a}, \mathbf{w}))
\end{equation}
\end{theorem}

\begin{proof}
    This result follows from the lemmas stated above and the result in \cite[Theorem 4]{young}. The stoachastically stable states of the game in Algorithm \ref{alg:matching} are contained in the recurrence classes of $P^0$ with minimum stochastic potential. We know from Lemma \ref{lem:not_stochastic_stable} that the stochastically stable state isa singleton contained in $\mathcal{Z}_{C,1}$. From Lemma \ref{lem:stochastic_potential}, we can get the action profile of the state with the minimum stochastic potential as:
    \begin{equation}
        (\mathbf{\Bar{a}}, \mathbf{w^0}) \in \mathop{\arg\min}\limits_{\mathbf{a}, \mathbf{w}} \gamma({z}_{C,1}) = \mathop{\arg\min}\limits_{\mathbf{a}, \mathbf{w}} \sum\limits_{k=1}^K (1 - u_k(\mathbf{a}, \mathbf{w})).
    \end{equation}
    Note the adversarial attack minimizing the potential will be $\mathbf{w^0}$. Thus, the action profile minimizing the stochastic potential is the same that achieves $\mathop{\max}\limits_{\mathbf{a}, \mathbf{w}} \sum\limits_{k=1}^K u_k(\mathbf{a}, \mathbf{w})$, and by the definition of $\mathcal{Z}_{C,1}$, the utilities are aligned with the baseline action and $\mathbf{w^0}$.
    \end{proof}

    Since we assume a unique optimal action profile, the state with the baseline actions and utilities corresponding to the optimal action profile and all players being content is the stochastically stable state. 

    In the following lemma, we bound the probability of the optimal action profile not being played during the exploitation phase of epoch $\ell$ (identified from the matching phase of epochs $\lceil\frac{\ell}{2}\rceil$ to $\ell$), given that event $E^\ell$ does not occur (\textit{i.e.}, the exploration phases of epochs $\lceil\frac{\ell}{2}\rceil$ to $\ell$ were successful).

    \begin{lemma}
In some epoch $\ell$, let 
\[
\mathbf{a}^*=\mathop{\arg\max}\limits_{\mathbf{a}, \mathbf{w^0} } \sum_{k=1}^K u_k(\mathbf{a}, \mathbf{w^0})
\]
and let $\mathbf{a'} = [a'_1,...,a'_K]$ where 
\[
a'_k = \mathop{\arg\max}\limits_{m \in [M]}  \sum_{i = \lceil\frac{\ell}{2}\rceil}^{\ell} W^i{(k,m)}
\]
is the action profile played in the exploitation phase of epoch $\ell$ by player $k$. 

Assume that for all players $k \in [K]$ and for all arms $m \in [M]$, the estimated mean rewards obtained at the end of the exploration phase for  epochs $\lceil\frac{\ell}{2}\rceil \leq i \leq \ell$ satisfy $|\hat{\mu}_k(m,1) - \mu_k(m,1)| \leq \Delta$. Then
%Then for small enough $\epsilon$, 
\[
P\{\mathbf{a^*} \neq \mathbf{a'}\} \leq \left(C_0 \exp{(- C_\rho \ell^{\delta})}\right)^\ell
\]
for some $C_0, C_\rho > 0$.
\end{lemma} 

The proof relies on using Chernoff-Hoeffding bounds for Markov chains (\cite[Theorem 3]{chung}) and is provided in Appendix \ref{apdx:proof_of_lemmas}.

\section{Preliminaries on Perturbed Markov Chains}\label{apdx:mabs_preliminaries}

In this section, we recap some preliminaries from \cite{young}, which as used to prove the results in this paper. 

Consider the recurrence classes of $P^0$ as $X_1 \ldots, X_L$. Then, we can define the resistance between two classes as the minimum resistance between any two states belonging
 to these classes, i.e.,
 \begin{equation}
     r(i,j) = \min\limits_{x \in X_i, y \in X_j} r(x,y), \text{ for } i, j \in [L].
 \end{equation}
 
Note that there is at least one path from every class to every other because $P^\epsilon$ is irreducible. We now define the fully connected graph  $G := ( \{1 \ldots, L\}, E_G)$. This graph has as vertex set the set of indices of the recurrence classes of $P^0$, and as edge set the set of directed edges between members of the recurrence classes. Also, $r(i,j)$ 
defines the resistance or weight of this directed edge.

Let an $i$-tree in $G$ be a spanning sub-tree of $G$, such that for every vertex $j \neq i$, there exists exactly one directed path from $j$ to $i$. Then, the stochastic potential $\gamma_i$ of the recurrence class $i$ is defined  as: 
\begin{equation}
    \gamma_i = \min\limits_{\tau \in \tau_i} \sum\limits_{(j_1,j_2) \in \tau} r(j_1,j_2).
\end{equation}
 where $\tau_i$ is the set of all $i$-trees in $G$.

The main result in \cite[Theorem 4]{young} states that the recurrence class $X_{i^*}$, with stochastic potential $\gamma_{i^*} = \min_{i = 1, \ldots, L} \gamma_i$ contains the stochastically stable states. We use this result to establish the guarantees of the matching phase of Algorithm \ref{alg:matching}.

\section{Proof of Lemmas}\label{apdx:proof_of_lemmas}

\begin{proof}[Proof of Lemma \ref{lem:perturbation}]
    To prove property (i) we first note that the process $P^\epsilon$ is irreducible for $\epsilon > 0$. Consider a state $z_d$ where all agents are discontent with some baseline action and utility. From such a state, the action of every user has full support in the set of actions or arms $[M]$, and each player can get content or discontent with a positive probability. Thus all states are accessible from the state $z_d$. On the other hand, starting from any state $z \in \mathcal{Z}$, the action of every user once again has full support in the set of actions or arms $[M]$, and thus one or more players can become discontent. This can lead to a state where all players can become discontent through the last step of the matching phase. Thus, all states are accessible from any state, and the Markov chain $p^\epsilon$ is irreducible.
    Additionally, it can be seen that the Markov chain is aperiodic, since the probability of remaining in state $z_d$ after one time instant is greater than zero. Thus $P^\epsilon$ is ergodic. 

    It is clear to see Property (ii) is satisfied by observing the transition probabilities. 

    Note that for the transitions with probability of the form of exponents of $\epsilon$, the resistance would be the exponent. For the transitions with the complementary probabilities, they are possible under $P^0$, and hence their resistance is $0$. Thus, property (iii) is satisfied and the process $P^\epsilon$ is a perturbed Markov process. 
\end{proof}

\begin{proof}[Proof of Lemma \ref{lem:recurrent_states}]
    Under $P^0$, the players in a state in $z_d \in \mathcal{Z}_D$, choose actions uniformly from the set of arms, and stay discontent in the state update step and the synchronising moods step. All the players stay discontent and move to another state in $\mathcal{Z}_D$. Thus, the state remains in the class $\mathcal{Z}_D$. 

    Under any singleton in $\mathcal{Z}_{C,i}$ for $i = 1, \ldots, K$, the players choose their same baseline action with probability one. Upon receiving their utilities, they stay content with the same baseline action and utility in the state update step. Since all players are content, they stay content in the last step of the matching phase as well. Thus, every singleton in $\mathcal{Z}_{C,i}$ is a recurrent class. 

    For any state that has a mixture of content and discontent players, all players become discontent in the last step of the matching phase. Thus, such a state is not recurrent under $P^0$. A state in which all players are content with none of the utilities of the players are aligned with their baseline actions cannot be reached through any transition step in the game. Thus, $\mathcal{Z}_D$ and the singletons in $\mathcal{Z}_{C,i}$ for $i = 1, \ldots, K$ are the only recurrence classes of $P^0$. 
\end{proof}

\begin{proof}[Proof of Lemma \ref{lem:resistances}]
The resistances for each of the state transitions are detailed below.
    \begin{itemize}
        \item In order to go from a state where all players are discontent to a state where all players are content, all the players need to become content with their utilities aligned with their action profile for some $\mathbf{w}$, and this occurs with the probability $\prod\limits_{k=1}^K \epsilon^{1 - u_{k}(\mathbf{a}, \mathbf{w})}$.
        
        \item This transition has to go through some state in $\mathcal{Z}_{C,1}$ and then at least $i$ players needs to explore with probability of order $\epsilon^\kappa$, thus providing a lower bound for the resistance.
        
        \item This transition can occur in multiple ways. It requires at lease one player to explore, and thus has a lower bound of $\kappa$. Once the player explores, the player can become content with probability $\epsilon^{1 - u_k}$ or can synchronise moods with the other players and become content with probability $\epsilon^\beta$. The players can also become discontent and become content again which incurs resistance $\kappa$. Thus, the upper bound for this resistance is $2\kappa$ (since $\beta$ < $\kappa$).
        
        \item This requires at least one player to explore and the players can all become discontent with resistance 0. Thus the resistance for this transition is $\kappa$. 

        \item This requires at least $i$ players needs to explore with probability of order $\epsilon^\kappa$, thus providing a lower bound for the resistance. 

        \item This requires at least one player to explore and the players can all become discontent with resistance 0. Thus the resistance for this transition is $\kappa$. 
        
        \item With the same reasoning as for the above transition, the lower bound for the resistance is obtained.

        \item This requires at least one player to explore and be content. Thus, we obtain the corresponding lower bound.

        \item The least resistant path for this transition requires at least $|i-j|$ players to explore. 
    \end{itemize}
\end{proof}

\begin{proof}[Proof of Lemma \ref{lem:not_stochastic_stable}]
    The proof idea here is similar to \cite[Lemma 4.5]{ramesh2016distributed}. The preliminaries on resistance trees are provided in Appendix \ref{apdx:mabs_preliminaries}. 
    
    Consider the tree rooted at a state $z_D$ achieving the stochastic potential for $\mathcal{Z}_D$, which must have a path from some state $z_{C,1}$ to $z_D$. This path has resistance $\kappa$ and thus the potential can be written as $\kappa + R_D$ where $R_D$ is the collective resistance of the other paths. Replacing this with a path from $z_{C,1}$ to $z_D$ gives a tree rooted at $z_{C,1}$ with potential $R_D + \sum\limits_{k=1}^K (1 - u_k) < R_D + \kappa$. Thus, $z_D$ cannot have the least potential tree.

    Similarly a tree rooted at $z_{C,i}$ for $i \geq 2$ must have a path through $z_D$, $z_{C,1}$, \ldots, $z_{C,i}$. This path has resistance greater than $\kappa i + \sum\limits_{k=1}^K(1 - u_k)$ and thus the potential can be written as $\kappa i + \sum\limits_{k=1}^K(1 - u_k) + R_i$ where $R_i$ is the collective resistance of the other paths. Replacing every link from $z_{C,j}$ to $z_{C,j+1}$ for $j = 1, \ldots i-1$ with a path from $z_{C,j+1}$ to some $z'_D$ gives a tree rooted at $z_{C,1}$ with potential exactly $\kappa i + \sum\limits_{k=1}^K(1 - u_k) + R_i$. Thus, $z_{C,i}$ for $i \geq 2$ cannot have the least stochastic potential.
\end{proof}

\begin{proof}[Proof of Lemma \ref{lem:stochastic_potential}]
    The proof of this lemma follows a similar idea as in \cite[Lemma 3]{marden}. 
    
    First we construct a tree rooted at $z_{C,1}$ that has the resistance $\gamma(z_{C,1})$. Consider a tree $T$ which has paths from $z'_{C,1} \in \mathcal{Z}_{C,1} \setminus{z_{C,1}}$ to $z_D \in \mathcal{Z}_{D}$ and paths from  $z_{C,i} \in \mathcal{Z}_{C,i}$ to $z_D$ and a path from $z_D$ to $z_{C,1}$. Note that all states $z_D$ are equivalent since they belong to the same recurrence class and thus have $0$ resistance between them. This tree has resistance $R(T) = \kappa \left( \sum\limits_{i = 1}^K |\mathcal{Z}_{C,i}| - 1\right) + \sum\limits_{k=1}^K (1 - u_k(\mathbf{a}, \mathbf{w})) $. Thus $\gamma(z_{C,1}) \leq R(T)$. 

    Consider any tree $T'$ rooted at $z_{C,1}$ that is not $T$ that has resistance less than $\gamma(z_{C,1})$. This tree $T'$ can have paths of the form $z_1 \to z_2 \to \ldots \to z_x \to z_D$ and $z_D \to z'_1 \to  \ldots \to z'_y \to z_{C,1}$ where $z_1, \ldots, z_x, z'_1, \ldots, z'_y \in \cup_{i=1}^K\mathcal{Z}_{C,i} \setminus{z_{C,1}}$. Each path of form  $z_1 \to z_2 \to \ldots \to z_x \to z_D$ has resistance greater than $x \kappa$. The links of the form $z_i \to z_j$ in a path of the form $z_1 \to z_2 \to \ldots \to z_x \to z_D$ from tree $T'$ can be replaced by $z_i \to z_D$  to form a tree $T''$. The replaced paths have a total resistance of $\kappa x$, thus $R(T'') \leq R(T')$. Starting from tree $T''$, the links starting from $z'_i$ in a path of the form $z_D \to z'_1  \to \ldots \to z'_y \to z_{C,1}$ from tree $T''$ can be replaced by $z'_i \to z_D$ with the link $z_D \to z_{C,1}$ added to form a tree $T'''$. Each path of form  $z_D \to z'_1  \to \ldots \to z'_y \to z_{C,1}$ has resistance greater than $y \kappa + \sum\limits_{k=1}^K (1 - u_k(\mathbf{a}, \mathbf{w})$ where $\mathbf{a}$ and $\mathbf{w}$ are the action profile and adversarial attack corresponding to $z_{C,1}$. The newly replaced paths have resistance exactly $y \kappa + \sum\limits_{k=1}^K (1 - u_k(\mathbf{a}, \mathbf{w})$. Thus, $R(T''') \leq R(T'') \leq R(T')$. We can continue this process to reach the tree $T$, i.e. until no such paths of the forms that $z_1 \to z_2 \to \ldots \to z_x \to z_D$ or $z_D \to z'_1 \to  \ldots \to z'_y \to z_{C,1}$ exist. Thus, $R(T) \leq R(T') < \gamma(z_{C,1}) \leq R(T)$ which is a contradiction. Thus, we have the result in Lemma \ref{lem:stochastic_potential}.
    \end{proof}

\begin{proof}[Proof of Lemma \ref{lem:matching_phase_success_prob}]
    The proof of this lemma follows the same idea as in \cite[Lemma 3]{magesh2021decentralized}.
    Since it is given that the exploration phases for all epochs $\lceil \frac{\ell}{2} \rceil$ to $\ell$ are successful, the efficient action profiles maximizing the sum of utilities under $\mathbf{w^0}$ in the matching phase for all epochs $\lceil \frac{\ell}{2} \rceil \leq i \leq \ell$ are the same, and are equal to the optimal action profile maximizing the sum of expected mean rewards:

$$\mathop{\arg\max}_{\mathbf{a} \in \mathcal{A}} \sum_{k=1}^{K} \mu_k(a_k,1) = \mathop{\arg\max}_{\mathbf{a} \in \mathcal{A}} \sum_{k=1}^{K} u_j(\mathbf{a}, \mathbf{w^0}).$$

Let $\mathbf{\Bar{u}^\ell}$ denote the utilities of the players for the optimal action profile $\mathbf{a^*}$ during epoch $\ell$. Note that $\mathbf{\Bar{u}^\ell}$ corresponds to the utilities with the action profile $\mathbf{a^*}$ and the adversarial attack $\mathbf{w^0}$. The optimal state of the Markov chain is then ${z^*}^\ell = [\mathbf{a^*},\mathbf{\Bar{u}^\ell},C^K ]$ during epoch $\ell$. Note that optimal state differs only in the baseline utilities for epochs $\lceil\frac{\ell}{2}\rceil$ to $\ell$.  In order to bound the probability of the event $\{\mathbf{a^*} \neq \mathbf{a'}\}$, we use the Chernoff Hoeffding bounds for Markov chains from \cite{chung}, which is also used in \cite{got} and \cite{bistritz2020game}. The function $f(z)$ considered here in order to use the bound from \cite[Theorem 3]{chung} for epoch $\ell$ is:

\begin{equation}
    f(z) = \identityf{{z} = {z^*}^{\ell}}, 
\end{equation}

\textit{i.e.} the state is the optimal state. Recall that $\tau_i = {c_2} i^{\delta}$ (replacing $\ell$ by $i$ in $\tau_\ell$). It follows that 

\begin{align}
    P\left\{\mathbf{a'} \neq \mathbf{a^*}\right\} \leq P\left\{\sum_{i = \lceil \frac{\ell}{2}\rceil}^\ell \sum_{t = 1}^{\tau_i} f(z_h) \leq \frac{1}{2}\sum_{i = \lceil \frac{\ell}{2}\rceil}^\ell \tau_i \right\}.
\end{align}

Define 

\begin{equation}
    X_i = \sum_{t = 1}^{\tau_i} \identityf{{z}_t = {z^*}^{i}}
\end{equation}

\begin{equation}
    L = \frac{1}{2}\sum_{i = \lceil \frac{\ell}{2}\rceil}^\ell \tau_i.
\end{equation}

Using the Markov inequality, for some $s > 0$, it follows that

\begin{align}
    P\left\{\sum_{i = \lceil \frac{\ell}{2}\rceil}^\ell X_i \leq \frac{1}{2}\sum_{i = \lceil \frac{\ell}{2}\rceil}^\ell \tau_i \right\} &= P\left\{e^{-s\sum_{i = \lceil \frac{\ell}{2}\rceil}^\ell X_i} \geq e^{-sL} \right\} \\
    &\leq e^{sL}\Pi_{i = \lceil \frac{\ell}{2}\rceil}^\ell E[{e^{-sX_i}}] \label{chernoff}.
\end{align}

In order to use the bound from \cite[Theorem 3]{chung}, we need to calculate $\mu^\ell = E{f(z)} = P({z} = {z^*}^\ell)$. 

Define 

\begin{align}
    \pi_z &= \min_{\lceil \frac{\ell}{2}\rceil \leq i \leq \ell} P\{{z} = {z^*}^i\} \\
    \mu &= \min_{\lceil \frac{\ell}{2}\rceil \leq i \leq \ell} \mu^i.
\end{align}

From the definition of a stochastically stable state, we can choose an $\epsilon$ small enough such that 

\begin{equation}
    \mu = \pi_z > \frac{1}{2(1-\eta)}
\end{equation}

for some $0 < \eta < 1/2$. 

We can now use the bound from \cite[Theorem 3]{chung} for epoch $\lceil \frac{\ell}{2}\rceil \leq i \leq \ell$ to get

\begin{align}
    P\left\{X_i \leq  \frac{\tau_i}{2} \right\} &\leq P\left\{\sum_{t = 1}^{\tau_i} \identityf{{z}_h = {z^*}^{i}} \leq (1 - \eta)\mu^i \tau_i \right\} \\
    &\leq c_0 \|\phi_i\|_\pi \exp\left(- \frac{\eta^2 \mu^i \tau_i}{72T}\right)\\
    &\leq c_0 \|\phi_i\|_\pi \exp\left(- \frac{\eta^2 \mu \tau_i}{72T}\right)
\end{align}

where $c_0 > 0$, $\phi_i$ is the initial distribution of the Markov chain in the $i$-th epoch and $T$ is the 1/8-th mixing time of the Markov chain. 
Using $s = \frac{\eta^2}{(1 - \eta)72T}$ it follows that,  

\begin{align}
    E{e^{-sX_i}} \leq (1 + c_0 \|\phi_i\|_\pi)  \exp\left(- \frac{\eta^2 \mu \tau_i}{72T} \right).
\end{align}

Using the above in \eqref{chernoff}, 

\begin{align}
     &P\left\{\sum_{i = \lceil \frac{\ell}{2}\rceil}^\ell X_i \leq \frac{1}{2}\sum_{i = \lceil \frac{\ell}{2}\rceil}^\ell \tau_i \right\} \\
     &\leq \Pi_{i = \lceil \frac{\ell}{2}\rceil}^\ell (1 + c_0 \|\phi_i\|_\pi) e^{\frac{\eta^2L}{(1 - \eta)72T}} e^{- \frac{\eta^2 \mu \tau_i}{72T}} \\
     &\leq C_0^\ell \exp \left({-\frac{\eta^2(\mu - \frac{1}{2(1 - \eta)})2L}{72T}}\right) \\ 
     &\leq C_0^\ell \exp \left({-\frac{{\eta^2 {c_2}} {2^{-(1 + \delta)}}(\mu - \frac{1}{2(1 - \eta)}) \ell^{1 + \delta}}{72T}}\right) \\
     &\leq \left(C_0 \exp \left( -C_\rho \ell^{\delta} \right)\right)^\ell
\end{align}

where $C_0 = \max_{\lceil \frac{\ell}{2}\rceil \leq i \leq \ell} (1 + c_0 \|\phi_i\|_\pi)$, and 
\begin{equation}\label{crho}
 C_\rho = \frac{{\eta^2 {c_2}} {2^{-(1 + \delta)}}(\mu - \frac{1}{2(1 - \eta)}) }{72T} > 0.   
\end{equation}

\end{proof}

\bibliographystyle{abbrv}
\bibliography{ref}

\end{document}